\documentclass[letterpaper, 10 pt, conference]{ieeeconf}  

\IEEEoverridecommandlockouts                              

\usepackage{amsmath,amssymb,bm,bbm,mathrsfs,amscd}
\usepackage{color}
\usepackage{dsfont}
\usepackage{graphicx}
\usepackage{epsfig}
\usepackage{subcaption}
\usepackage{algorithm}
\usepackage[noend]{algpseudocode}
\usepackage{tikz}

\newtheorem{theorem}{Theorem}
\newtheorem{definition}{Definition}

\newtheorem{lemma}{Lemma}
\newtheorem{corollary}{Corollary}
\newtheorem{remark}{Remark}
\newtheorem{assumption}{Assumption}

\newcommand{\RR}{{\mathbb{R}}}
\newcommand{\NN}{{\mathbb{N}}}
\newcommand{\EE}{{\mathbb{E}}}
\newcommand{\PP}{{\mathbb{P}}}
\newcommand{\JJ}{{\mathbb{J}}}
\newcommand{\FF}{{\mathbb{F}}}

\newcommand{\mc}{\mathcal}
\newcommand{\norm}[1]{\|#1\|}
\newcommand{\normsq}[1]{\|#1\|^2}
\newcommand{\EEk}[1]{\EE[#1|\mc F_k]}

\newcommand{\op}{\operatorname}

\newcommand{\bs}{\boldsymbol}
\newcommand{\fineass}{\hfill\small$\blacksquare$}

\title{\LARGE \bf
A game--theoretic approach for Generative Adversarial Networks
}

\author{Barbara Franci and Sergio Grammatico
\thanks{The authors are with the Delft Center for System and Control, TU Delft, The Netherlands
        {\tt\footnotesize \{b.franci-1, s.grammatico\}@tudelft.nl}}%
\thanks{This work was partially supported by NWO under research projects OMEGA (613.001.702) and P2P-TALES (647.003.003), and by the ERC under research project COSMOS (802348).}}

\begin{document}

\maketitle
\thispagestyle{empty}
\pagestyle{empty}

\begin{abstract}
Generative adversarial networks (GANs) are a class of generative models, known for producing accurate samples. The key feature of GANs is that there are two antagonistic neural networks: the generator and the discriminator. The main bottleneck for their implementation is that the neural networks are very hard to train. One way to improve their performance is to design reliable algorithms for the adversarial process. Since the training can be cast as a stochastic Nash equilibrium problem, we rewrite it as a variational inequality and introduce an algorithm to compute an approximate solution.
Specifically, we propose a stochastic relaxed forward--backward algorithm for GANs. We prove that when the pseudogradient mapping of the game is monotone, we have convergence to an exact solution or in a neighbourhood of it.
\end{abstract}

\section{Introduction}

Generative adversarial networks (GANs) are an example of generative models. Specifically, the model takes a training set, consisting of samples drawn from a probability distribution, and learns how to represent an estimate of that distribution \cite{goodfellow2016}. GANs focus primarily on sample generation, but it is also possible to design GANs that can estimate the probability distribution explicitly \cite{goodfellow2016}.
The subject has been recently studied, especially because it has many practical applications on various topics. For instance, they can be used for medical purposes, i.e., to improve the diagnostic performance for the low-dose computed tomography method
\cite{yang2018}, for polishing images taken in unfavourable weather conditions (as rain or snow)
\cite{zhang2019}. Other applications range from
speech and language recognition, to playing chess and vision computing \cite{wang2017}.

The idea behind GANs is to train the generative model via an adversarial process, in which also the opponent is simultaneously trained. Therefore, there are two neural network classes: a generator that captures the data distribution, and a discriminator that estimates the probability that a sample came from the training data rather than from the generator. 
The generative model can be thought of as a team of counterfeiters, trying to produce fake currency, while the discriminative model, i.e., the police, tries to detect the counterfeit money \cite{goodfellow2016}. The competition drives both teams to improve their methods until the counterfeit currency is indistinguishable from the original.
To succeed in this game, the counterfeiter must learn to make money that are indistinguishable from original currency, i.e. the generator network must learn to create samples that are drawn from the same distribution as the training data \cite{goodfellow2014}.

Since each agent payoff depends on the other agent, the problem can be described as a game. Therefore, these networks are called \textit{adversarial}. However, GANs can be also thought as a game with cooperative players since they share information with each other \cite{goodfellow2016} and since there are only the generator and the discriminator, it is an instance of a two-player game. Moreover, depending on the cost functions, it can also be considered as a zero-sum game. From a mathematical perspective, the class of games that suits the GAN problem is that of stochastic Nash equilibrium problems (SNEPs) where each agent aims at minimizing its expected value cost function, approximated via a number of samples of the random variable.

Given their connection with optimization and game theory, GANs have received theoretical attention as well, for modelling as Nash equilibrium problems \cite{oliehoek2017,mazumdar2020} and for designing algorithms to improve the training process \cite{gidel2019,mazumdar2020}.

From a game theoretic perspective, an elegant approach to compute a SNE is to cast the problem as a stochastic variational inequality (SVI) \cite{facchinei2007} and to use an iterative algorithm to find a solution.
The two most used methods for SVIs studied for GANs are the gradient method \cite{bruck1977}, known also as forward--backward (FB) algorithm \cite{robbins1951}, and the extragradient (EG) method \cite{iusem2017}. The iterates of the FB algorithm involve an evaluation of the pseudogradient and a projection step. They are known to converge if the pseudogradient mapping is cocoercive or strongly monotone \cite{rosasco2016,franci2019}. Such technical assumptions are quite strong if we consider that in GANs the mapping is rarely monotone. In contrast, the EG algorithm converges for merely monotone operators but taking two projections into the local constraint set per iteration, thus making the algorithm computationally expensive. 

In this paper we propose a stochastic relaxed FB (SRFB) algorithm, inspired by \cite{malitsky2019}, for GANs. A first analysis of the algorithm for stochastic (generalized) NEPs is currently under review \cite{franci2020}. The SRFB requires a single projection and single evaluation of the pseudogradient algorithm per iteration. The advantage of our proposed algorithm is that it is less computationally demanding than the EG algorithm even if it converges under the same assumptions. Indeed, we prove its convergence under mere monotonicity of the pseudogradient mapping when a huge number of samples is available. Alternatively, if only a finite number of samples is accessible, we prove that averaging can be used to converge to a neighbourhood of the solution. 



\section{Generative Adversarial Networks}

The basic idea of generative adversarial networks (GANs) is to set up a game between two players: the generator and the discriminator. 
The generator creates samples that are intended to come from the same distribution as the training data. 
The discriminator examines the samples to determine whether they are real or fake. 
The generator is therefore trained to fool the discriminator. 
Typically, a deep neural network is used to represent the generator and the discriminator. Accordingly, the two players are denoted by two functions, each of which is differentiable both with respect to its inputs and with respect to its parameters.

The generator is represented by a differentiable function $g$, that is, a neural network class with parameter vector $x_{\text{g}} \in \Omega_{\text{g}}\subseteq\RR^{n_{\text{g}}}$. The (fake) output of the generator is denoted with $g(z,x_{\text{g}}) \in \mathbb{R}^{q}$ where the input $z$ is a random noise drawn from the model prior distribution, $z \sim p_{\text{z}}$, that the generator uses to create the fake output $g(z,x_{\text{g}})$ \cite{oliehoek2017}.
The actual strategies of the generator are the parameters $x_{\text{g}}$ that allows $g$ to produce the fake output.

The discriminator is a neural network class as well, with parameter vector $x_{\text{d}} \in \Omega_{\text{d}}\subseteq\RR^{n_{\text{d}}}$ and a single output $d(v,x_{\text{d}}) \in[0,1]$ that indicates the accuracy of the input $v$. We interpret the output as the probability that the discriminator $d$ assigns to an element $v$ to be real. Similarly to the generator $g$, the strategies of the discriminator are the parameters $x_{\text{d}}$.

The problem can be cast as a two player game, or, depending on the cost functions, as a zero sum game. Specifically, in the latter case the mappings $J_\text{g}$ and $J_\text{d}$ should satisfy the following relation 
\begin{equation}\label{eq_rel_zero}
J_{\text{g}}(x_{\text{g}},x_{\text{d}})=-J_{\text{d}}(x_{\text{g}},x_{\text{d}}).
\end{equation}

Often \cite{gidel2019,goodfellow2014}, the payoff of the discriminator is given by
\begin{equation}\label{eq_payoff_zero}
J_{\text{d}}(x_{\text{g}},x_{\text{d}})=\EE[\phi(d(\cdot,x_{\text{d}})]-\EE[\phi(d(g(\cdot,x_{\text{g}}),x_{\text{d}}))]
\end{equation}
where $\phi:[0,1] \rightarrow \mathbb{R}$ is a measuring function (typically a logarithm \cite{goodfellow2014}). The mapping in \eqref{eq_payoff_zero} can be interpreted as the distance between the real value and the fake one.

In the context of zero sum games, the problem can be rewritten as a minmax problem
\begin{equation}\label{eq_minmax}
\min_{x_{\text{g}}}\max_{x_{\text{d}}} J_{\text{d}}(x_{\text{g}},x_{\text{d}}).
\end{equation}
In words, \eqref{eq_minmax} means that the generator aims at minimizing the distance from the real value while the discriminator wants to maximize it, i.e. to recognize the fake data.

When the problem is not a zero sum game, the generator has its own cost function, usually given by \cite{gidel2019}
\begin{equation}\label{eq_cost_gen}
J_{\text{g}}(x_{\text{g}},x_{\text{d}})=\EE[\phi(d(g(\cdot,x_{\text{g}}),x_{\text{d}}))].
\end{equation}
It can be proven that the two-player game with cost functions \eqref{eq_payoff_zero} and \eqref{eq_cost_gen} and the zero-sum game with cost function \eqref{eq_payoff_zero} and relation \eqref{eq_rel_zero} are strategically equivalent \cite[Th. 10]{oliehoek2017}.

\section{Stochastic Nash equilibrium problems}

In this section we formalize the two player game in a more general form, since our analysis is independent on the choice of the cost functions. Namely, we consider the problem as a stochastic Nash equilibrium problem.

We consider a set of two agents $\mc I=\{\text{g},\text{d}\}$, that represents the two neural network classes.
The local cost function of agent $i\in \mc I$ is defined as 
\begin{equation}\label{eq_cost_stoc}
\JJ_i(x_i,x_j)=\EE_\xi[J_i(x_i,x_j,\xi(\omega))],
\end{equation}
for some measurable function $J_i:\mc \RR^{n}\times \RR^q\to \RR $ where $n=n_{\text{d}}+n_{\text{g}}$. The cost function $\JJ_i $ of agent $i\in\mc I $ depends on the local variable $x_i\in\Omega_i\subseteq\RR^{n_i}$, the decisions of the other player $x_j$, $j\neq i$, and the random variable $\xi:\Xi\to\RR^q $ that express the uncertainty. Such uncertainty arises in practice when it is not possible to have access to the exact mapping, i.e., when only a finite number of estimates are available.
 $\EE_\xi $ represent the mathematical expectation with respect to the distribution of the random variable $\xi(\omega) $\footnote{From now on, we use $\xi $ instead of $\xi(\omega) $ and $\EE $ instead of $\EE_\xi $.} in the probability space $(\Xi, \mc F, \PP) $. We assume that $\EE[J_i(\bs{x},\xi)] $ is well defined for all the feasible $\bs{x}=\op{col}(x_{\text{g}},x_{\text{d}})\in\bs \Omega=\Omega_{\text{g}}\times\Omega_{\text{d}} $ \cite{ravat2011}. 
For our theoretical analysis, we postulate the following assumptions on the cost function and on the feasible set \cite{ravat2011}.
\begin{assumption}\label{ass_J}
For each $i,j \in \mc I $, $i\neq j$, the function $\JJ_{i}(\cdot, x_j) $ is convex and continuously differentiable.
\fineass\end{assumption}
\begin{assumption}
For each $i \in \mc I,$ the set $\Omega_{i}$ is nonempty, compact and convex.
\fineass\end{assumption}

Given the decision variables of the other agent, each player $i $ aims at choosing a strategy $x_i $, that solves its local optimization problem, i.e.,
\begin{equation}\label{eq_game}
\forall i \in \mc I: \quad
\min\limits _{x_i \in \Omega_i}  \JJ_i\left(x_i, x_j\right).\\ 
\end{equation}
Given the coupled optimization problems in \eqref{eq_game}, the solution concept that we are seeking is that of stochastic Nash equilibrium (SNE) \cite{ravat2011}.
\begin{definition}\label{def_GNE}
A stochastic Nash equilibrium is a collective strategy $\bs x^*=\op{col}(x_{\text{g}}^*,x_{\text{d}}^*)\in\bs{\Omega} $ such that for all $i \in \mc I $
 $$\JJ_i(x_i^{*}, x_j^{*}) \leq \inf \{\JJ_i(y, x_j^{*})\; | \; y \in \Omega_i\}. $$
\end{definition}
Thus, a SNE is a set of actions where no agent can decrease its cost function by unilaterally changing its decision.

To have theoretical guarantees that a SNE exists, we make further assumptions on the cost functions \cite[Ass. 1]{ravat2011}.
\begin{assumption}\label{ass_J_exp}
For each $i\in\mc I $ and for each $\xi \in \Xi $, the function $J_{i}(\cdot,x_j,\xi) $ is convex, Lipschitz continuous, and continuously differentiable. The function $J_{i}(x_i,x_j,\cdot) $ is measurable and for each $\bs x$ and its Lipschitz constant $\ell_i(x_j,\xi) $ is integrable in $\xi $.
\fineass\end{assumption}

Existence of a SNE of the game in \eqref{eq_game} is guaranteed, under Assumptions \ref{ass_J}-\ref{ass_J_exp}, by \cite[\S 3.1]{ravat2011} while uniqueness does not hold in general \cite[\S 3.2]{ravat2011}.

For seeking a Nash equilibrium, we rewrite the problems as a stochastic variational inequality. 
To this aim, let us denote the pseudogradient mapping as
\begin{equation}\label{eq_grad}
\FF(\bs{x})=\left[\begin{array}{c}
\EE[\nabla_{x_{\text{g}}} J_{\text{g}}(x_{\text{g}}, x_{\text{d}})]\\
\EE[\nabla_{x_{\text{d}}} J_{\text{d}}(x_{\text{d}}, x_{\text{g}})]
\end{array}\right],
\end{equation} 
where the possibility to exchange the expected value and the pseudogradient in \eqref{eq_grad} is assured by Assumption \ref{ass_J_exp}. Then, the associated stochastic variational inequality (SVI) reads as
\begin{equation}\label{eq_SVI}
\langle \FF(\bs x^*),\bs x-\bs x^*\rangle\geq 0\text { for all } \bs x \in \bs{\Omega}.
\end{equation}


\begin{remark}
If Assumptions \ref{ass_J}--\ref{ass_J_exp} hold, then a tuple $\bs x^*\in\bs{\Omega}$ is a Nash equilibrium of the game in \eqref{eq_game} if and only if $\bs x^*$ is a solution of the SVI in \eqref{eq_SVI} \cite[Prop. 1.4.2]{facchinei2007}, \cite[Lem. 3.3]{ravat2011}. We call these equilibria, variational equilibria (v-SNE).

Moreover, under Assumptions \ref{ass_J}-\ref{ass_J_exp}, the solution set of $\op{SVI}(\bs{\Omega},\FF) $ is non empty and compact, i.e. $\op{SOL}(\bs{\Omega},\FF)\neq\varnothing $ \cite[Cor. 2.2.5]{facchinei2007} and a v-SNE exists.
\fineass\end{remark}


\section{Stochastic relaxed forward--backward with averaging}

The first algorithm that we propose (Algorithm \ref{algo_i}) is inspired by \cite{malitsky2019,franci2020} and it is a stochastic relaxed forward backward algorithm with averaging (aSRFB). 
\begin{algorithm}
\caption{Stochastic Relaxed Forward--Backward with averaging (aSRFB)}\label{algo_i}
Initialization: $x_i^0 \in \Omega_i$\\
Iteration $k\in\{1,\dots,K\}$: Agent $i\in\{\text g,\text d\}$ receives $x_j^k$, $j \neq i$, then updates:
\begin{subequations}
\begin{align}
\bar{x}_i^{k} &=(1-\delta) x_i^k+\delta\bar{x}_i^{k-1} \label{eq_relax}\\ 
x_i^{k+1}&=\op{proj}_{\Omega_i}[\bar{x}_i^k-\lambda_{i}F^{\textup{SA}}_{i}(x_i^k, x_j^k,\xi_i^k)]
\end{align}
\end{subequations}
Iteration $K$: 
\begin{equation}\label{eq_ave}
X^K_i=\frac{1}{K}\sum_{k=1}^K\bs x^k
\end{equation}
\end{algorithm}

We note that the averaging step in \eqref{eq_ave}
was first proposed for VIs in \cite{bruck1977}, and it can be implemented in an online fashion as 
\begin{equation}\label{eq_ave_online}
\bs X^K=(1-\tilde{\lambda}_{K})\bs X^{K-1}+\tilde{\lambda}_{K} \bs x^{K}
\end{equation}
where $0 \leq \tilde{\lambda}_{K} \leq 1$ and $\bs X^K=\op{col}(X_\text{g}^K,X_\text{d}^K)$, . We note that \eqref{eq_ave_online} is different from \eqref{eq_relax}. Indeed, \eqref{eq_relax} is a convex combination of the two previous iterates $\bs x^k$ and $\bar{\bs x}^{k-1}$, updated at every time $k$ and with a fixed parameter $\delta$, while the averaging in \eqref{eq_ave_online} is a weighted cumulative sum over all the decision variables $\bs x^k$ until time $K$ with time varying weights $\tilde\lambda$, $k\in\{1,\dots,K\}$.
The parameter $\tilde{\lambda}_K$ can be tuned to obtain uniform, geometric or exponential averaging \cite{gidel2019}.
The relaxation parameter $\delta$ instead should be bounded.

\begin{assumption}\label{ass_delta1}
In Algorithm \ref{algo_i}, $\delta\in[0,1]$.\fineass
\end{assumption}

To continue our analysis, we postulate the following monotonicity assumption on the pseudogradient mapping. This assumption is among the weakest possible for SVIs \cite{iusem2017,malitsky2019} and it is common also for GANs \cite{gidel2019}.
\begin{assumption}\label{ass_mono}
$\FF$ as in \eqref{eq_grad} is monotone, i.e. $\langle\FF(\bs x)-\FF(\bs y),\bs x-\bs y\rangle\geq0$ for all $\bs x,\bs y\in\bs \Omega$.
\fineass\end{assumption}

Next, let us define the stochastic approximation of the pseudogradient \cite{robbins1951} as
\begin{equation}\label{eq_SA}
F^{\textup{SA}}(x,\xi)=\left[\begin{array}{c}
\nabla_{x_{\text{g}}} J_{\text{g}}(x_{\text{g}}, x_{\text{d}},\xi_\text{g})\\
\nabla_{x_{\text{d}}} J_{\text{d}}(x_{\text{d}}, x_{\text{g}},\xi_\text{d})
\end{array}\right].
\end{equation}
$F_i^{\textup{SA}}, i\in\mc I$, uses one or a finite number (mini-batch) of realizations of the random variable $\xi_i$. Given the approximation, we postulate the following assumption which is quite strong yet reasonable in our game theoretic framework \cite{gidel2019}. Let us first define the filtration $\mc F=\{\mc F_k\}$, that is, a family of $\sigma$-algebras such that $\mathcal{F}_{0} = \sigma\left(X_{0}\right)$ and 
$\mathcal{F}_{k} = \sigma\left(X_{0}, \xi_{1}, \xi_{2}, \ldots, \xi_{k}\right)$ for all $k \geq 1,$
such that $\mc F_k\subseteq\mc F_{k+1}$ for all $k\geq0$. 
\begin{assumption}\label{ass_bounded}
$F^{\textup{SA}}$ in \eqref{eq_SA} is bounded, i.e., there exists $B>0$ such that for $\bs x\in \bs\Omega$, 
$\EE[\|F^{\textup{SA}}(\bs x,\xi)\|^2|\mc F_k]\leq B.$
\fineass\end{assumption}

For the sake of our analysis, we make an explicit bound on the feasible set.
 
\begin{assumption}\label{ass_boundedset}
The local constraint set $\bs\Omega$ is such that $\max_{\bs x,\bs y\in \bs\Omega} \|\bs x-\bs y\|^{2} \leq R^{2}$, for some $R\geq0$.
\fineass\end{assumption}

For all $k \geq 0,$ we define the stochastic error as
\begin{equation}\label{eq_error}
\epsilon_k= F^{\textup{SA}}(\bs x^k,\xi^k)-\FF(\bs x^k),
\end{equation}
i.e., the distance between the approximation and the exact expected value. Then, the stochastic error should have zero mean and bounded variance, as usual in SVI \cite{gidel2019,iusem2017}.


\begin{assumption}\label{ass_error}
The stochastic error in \eqref{eq_error} is such that, for all $k\geq 0$,
$\EE[\epsilon^k|\mc F_k]=0 \text{ and }\EE[\|\epsilon^k\|^2|\mc F_k]\leq\sigma^2$ a.s..
\fineass\end{assumption}


As a measure of the quality of the solution, we define the following error
\begin{equation}\label{gap}
\operatorname{err}(\bs x)=\max _{\bs x^* \in\bs \Omega} \langle\FF(\bs x^*),\bs x-\bs x^*\rangle,
\end{equation}
which is known as gap function and it is equal 0 if and only if $\bs x^*$ is a solution of the (S)VI in \eqref{eq_SVI} \cite[Eq. 1.5.2]{facchinei2007}. 
Other possible measures can be found in \cite{gidel2019}.

We are now ready to state our first result.
\begin{theorem}\label{theo_ave}
Let Assumptions \ref{ass_J}-\ref{ass_error} hold. Let $\bs X^K$ be as in \eqref{eq_ave}, $c=\frac{2-\delta^2}{1-\delta}$ and $B$, $R$ and $\sigma^2$ as in Assumptions \ref{ass_bounded}-\ref{ass_error}.
Then Algorithm \ref{algo_i} with
$F^{\textup{SA}}$ as in \eqref{eq_SA} gives
$$\EE[\op{err}(\bs X^K)]=\frac{cR}{\lambda K}+(2B^2+\sigma^2)\lambda.$$
Thus,
$\lim_{K\to\infty}\EE[\op{err}(\bs X^K)]=(2B^2+\sigma^2)\lambda.$
\end{theorem}
\begin{proof}
See Appendix \ref{app_theo_ave}.
\end{proof}


\section{Sample average approximation}
If a huge number of samples is available or it is possible to compute the exact expected value, one can consider using a different approximation scheme or a deterministic algorithm. 

In the SVI framework, using a finite, fixed number of samples is called stochastic approximation (SA). It is widely used in the literature but it often requires the step sizes to 
be diminishing, with the results that the iterations slow down considerably. The approach that is used to keep a fixed step size is the sample average approximation (SAA) scheme. In this case, an increasing number of samples is taken at each iteration and this helps having a diminishing error.

With the SAA scheme, it is possible to prove convergence to the exact solution without using the averaging step. We show this result in Theorem \ref{theo_SAA} but first we provide more details on the approximation scheme and state some assumptions. The algorithm that we are proposing is presented in Algorithm \ref{algo_SAA_i}. The differences with Algorithm \ref{algo_i} are the absence of the averaging step and the approximation $F^{\textup{SAA}}$.
\begin{algorithm}
\caption{Stochastic Relaxed Forward--Backward (SRFB)}\label{algo_SAA_i}
Initialization: $x_i^0 \in \Omega_i$\\
Iteration $k$: Agent $i$ receives $x_j^k$ for $j \neq i$, then updates:
\begin{subequations}
\begin{align}
\bar{x}_i^{k} &=(1-\delta) x_i^k+\delta\bar{x}_i^{k-1} \\ 
x_i^{k+1}&=\op{proj}_{\Omega_i}[\bar{x}_i^k-\lambda_{i}F^{\textup{SAA}}_{i}(x_i^k, x_j^k,\xi_i^k)]
\end{align}
\end{subequations}
\end{algorithm}

Formally, the approximation that we use is given by
\begin{equation}\label{eq_SAA}
F^{\textup{SAA}}(\bs x,\xi^k)=\left[\begin{array}{c}
\frac{1}{N_k}\sum_{s=1}^{N_k}\nabla_{x_{\text{g}}} J_i(x_{\text{g}}^k,x_{\text{d}}^k,\xi_{\text{g}}^{(s)})\\
\frac{1}{N_k}\sum_{s=1}^{N_k}\nabla_{x_{\text{d}}} J_i(x_{\text{d}}^k,x_{\text{g}}^k,\xi_{\text{d}}^{(s)})
\end{array}\right]
\end{equation}
where $N_k$ is the batch size that should be increasing \cite{iusem2017}.
\begin{assumption}\label{ass_batch}
The batch size sequence $(N_k)_{k\geq 1}$ is such that
$N_k\geq b(k+k_0)^{a+1},$ for some $b,k_0,a>0$.
\fineass\end{assumption}
With a little abuse of notation, let us denote the stochastic error also in this case as
$$\epsilon^k=F^{\textup{SAA}}(\bs x^k,\xi^k)-\FF(\bs x^k).$$
\begin{remark}\label{remark_error}
Using the SAA scheme, it is possible to prove that, for some $C>0$, $\EEk{\normsq{\epsilon_k}}\leq\frac{C\sigma^2}{N_k},$
i.e., the error diminishes as the size of the batch increases. Details on how to obtain this result can be found in \cite[Lem. 3.12]{iusem2017}.
\fineass\end{remark}

To obtain convergence, we have to make further assumptions on the pseudogradient mapping \cite{malitsky2019,iusem2017}.
\begin{assumption}\label{ass_lip}
$\FF$ as in \eqref{eq_grad} is $\ell$-Lipschitz continuous for $\ell>0$, i.e., $\norm{\FF(\bs x)-\FF(\bs y)}\leq\ell\norm{\bs x-\bs y}$ for all $\bs x,\bs y\in \bs\Omega$.
\fineass\end{assumption}

The relaxation parameter should not be too small.
\begin{assumption}\label{ass_delta}
In Algorithm \ref{algo_SAA_i}, $\delta\in[\frac{1+\sqrt{5}}{2},1]$. 
\fineass\end{assumption}

Conveniently, with the SAA scheme we can take a constant step size, as long as it is small enough.
\begin{assumption}\label{ass_step1}
The steps size is such that $\lambda\in(0, \frac{1}{2\delta(2\ell+1)}]$
where $\ell$ is the Lipschitz constant of $\FF$ as in Assumption \ref{ass_lip}.
\fineass\end{assumption}

We can finally state our convergence result.
\begin{theorem}\label{theo_SAA}
Let Assumptions \ref{ass_J}-\ref{ass_mono} and \ref{ass_error}-\ref{ass_step1} hold. Then, the sequence $(\bs x^k)_{k\in\NN}$ generated by Algorithm \ref{algo_SAA_i} with $F^{\textup{SAA}}$ as in \eqref{eq_SAA} converges a.s. to an SNE of the game in \eqref{eq_game}.  
\end{theorem}
\begin{proof}
See Appendix \ref{app_theo_SAA}.
\end{proof}

If one is able to compute the exact expected value, the problem is equivalent to the deterministic case. Convergence follows under the same assumptions made for the SAA scheme with the exception of those on the stochastic error.

\begin{corollary}
Let Assumptions \ref{ass_J}-\ref{ass_mono} and \ref{ass_batch}-\ref{ass_step1} hold. Then, the sequence $(\bs x^k)_{k\in\NN}$ generated by Algorithm \ref{algo_SAA_i} with $\FF$ as in \eqref{eq_grad} converges a.s. to a solution of the game in \eqref{eq_game}.
\end{corollary}
\begin{proof}
It follows from Theorem \ref{theo_SAA} or \cite[Th. 1]{malitsky2019}.
\end{proof}

\section{Numerical simulations}
In this section, we present some numerical experiments to validate the analysis. We propose two theoretical comparison between the most used algorithms for GANs \cite{gidel2019}. 
In both the examples, we simulate our SRFB algorithm, the SFB algorithm \cite{rosasco2016}, the EG algorithm \cite{iusem2017}, the EG algorithm with extrapolation from the past (PastEG) \cite{gidel2019} and Adam, a typical algorithm for GANs \cite{kingma2014}. 

All the simulations are performed on Matlab R2019b with a 2,3 GHz Intel Core i5 and 8 GB LPDDR3 RAM.

\subsection{Illustrative example}\label{ex_academic}
In order to make a comparison, we consider the following zero-sum game which is a problematic example, for instance, for the FB algorithm \cite[Prop. 1]{gidel2019}. 

We suppose that the two players aims at solving the minmax problem in \eqref{eq_minmax} with cost function 
$$\JJ(x_{\text{g}},x_{\text{d}})=x_\text{g}^{\top} M(\xi) x_\text{d}+x_\text{g}^{\top} a+x_\text{d}^{\top} b.$$
The matrix $M(\xi)\in\RR^{n_\text{g}\times n_\text{d}}$ is the stochastic part that we approximate with the SAA scheme. $M(\xi)$ is an antidiagonal matrix, i.e., $M_{i,j}(\xi)\neq0$ if and only if $i+j=n+1$, and the entries are sampled from a normal distribution with mean 1 and finite variance. The mapping is monotone and $a\in\RR^{n_\text{g}}$ and $b\in\RR^{n_\text{d}}$. The problem is constrained so that $x_i\in[-1,1]^{n_i}$ and the optimal solution is $(b,-a)$. The step sizes are taken according to Assumption \ref{ass_step1}.

We plot the distance from the solution, the distance of the average from the solution, and the computational cost in Figure \ref{plot_sol}, \ref{plot_ave} and \ref{plot_sec}, respectively.

As one can see from Fig. \ref{plot_sol}, the SFB does not converge. From Fig. \ref{plot_sec} instead, we note that the SRFB algorithm is the less computationally expensive. Interestingly, the average tends to smooth the convergence to a solution.

\begin{figure}
\begin{subfigure}{\columnwidth}
\centering
\includegraphics[width=.9\textwidth]{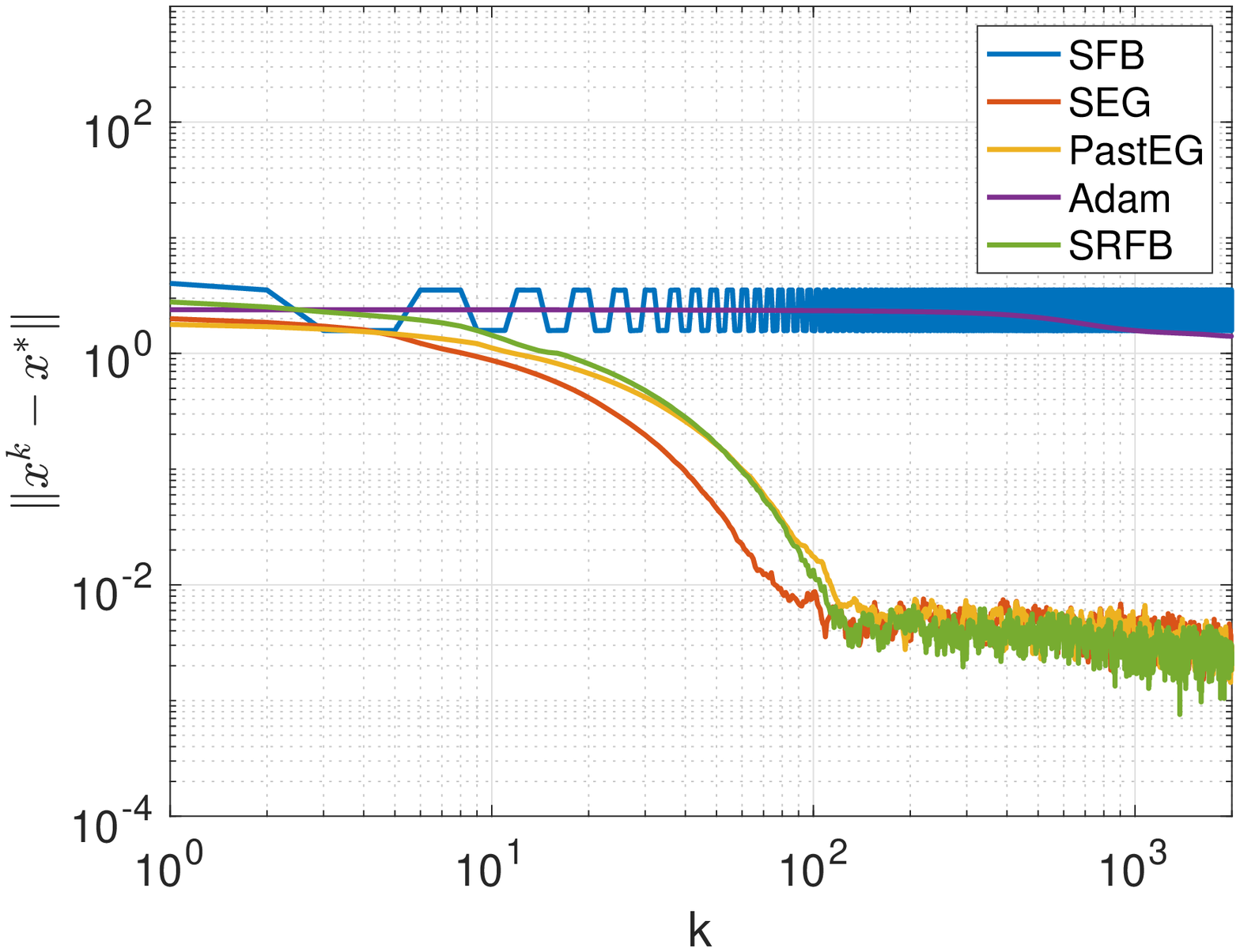}
\caption{Relative distance from the solution.}\label{plot_sol}
\end{subfigure}
\begin{subfigure}{\columnwidth}
\centering
\includegraphics[width=.9\textwidth]{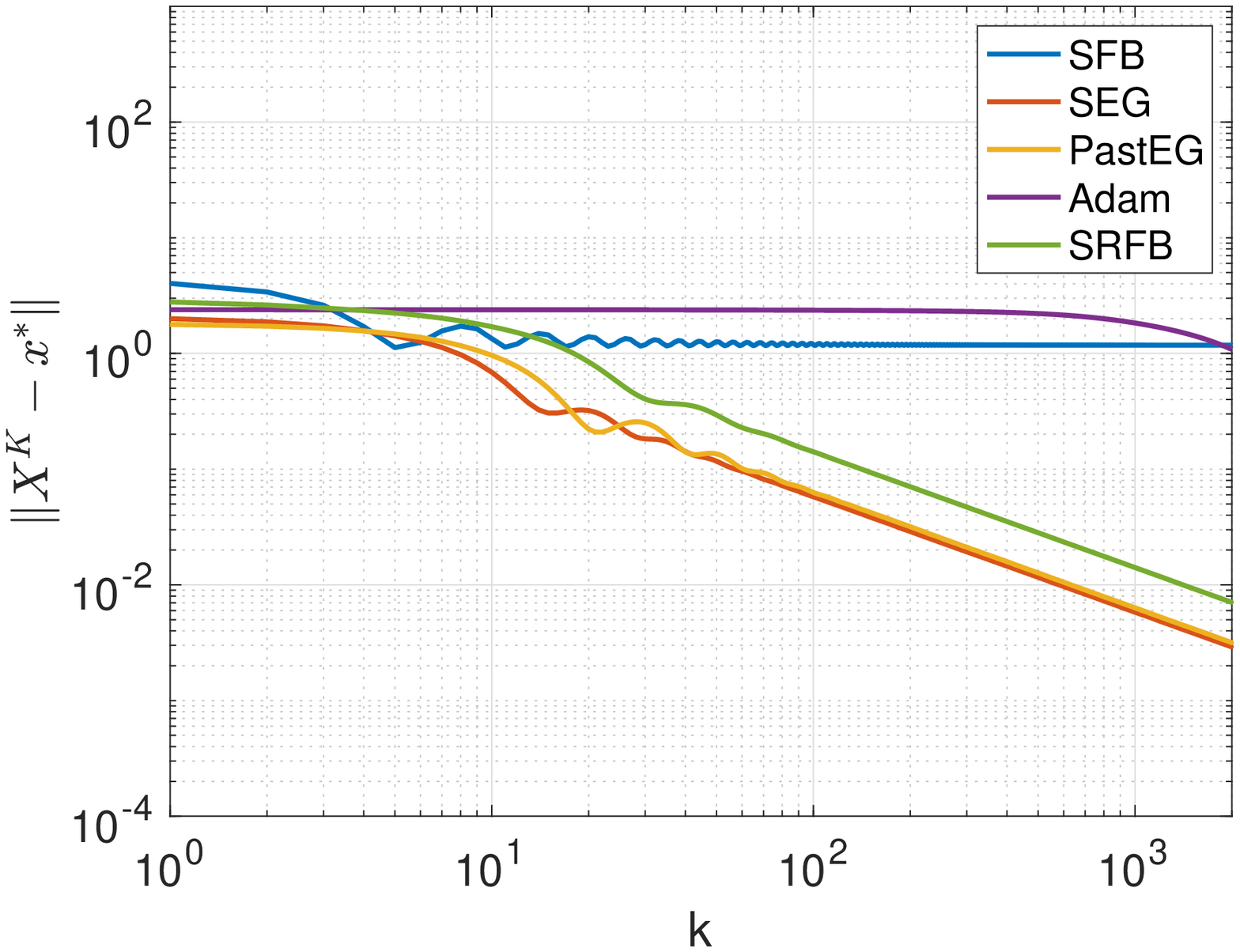}
\caption{Relative distance of the average from the solution.}\label{plot_ave}
\end{subfigure}
\begin{subfigure}{\columnwidth}
\centering
\includegraphics[width=.9\textwidth]{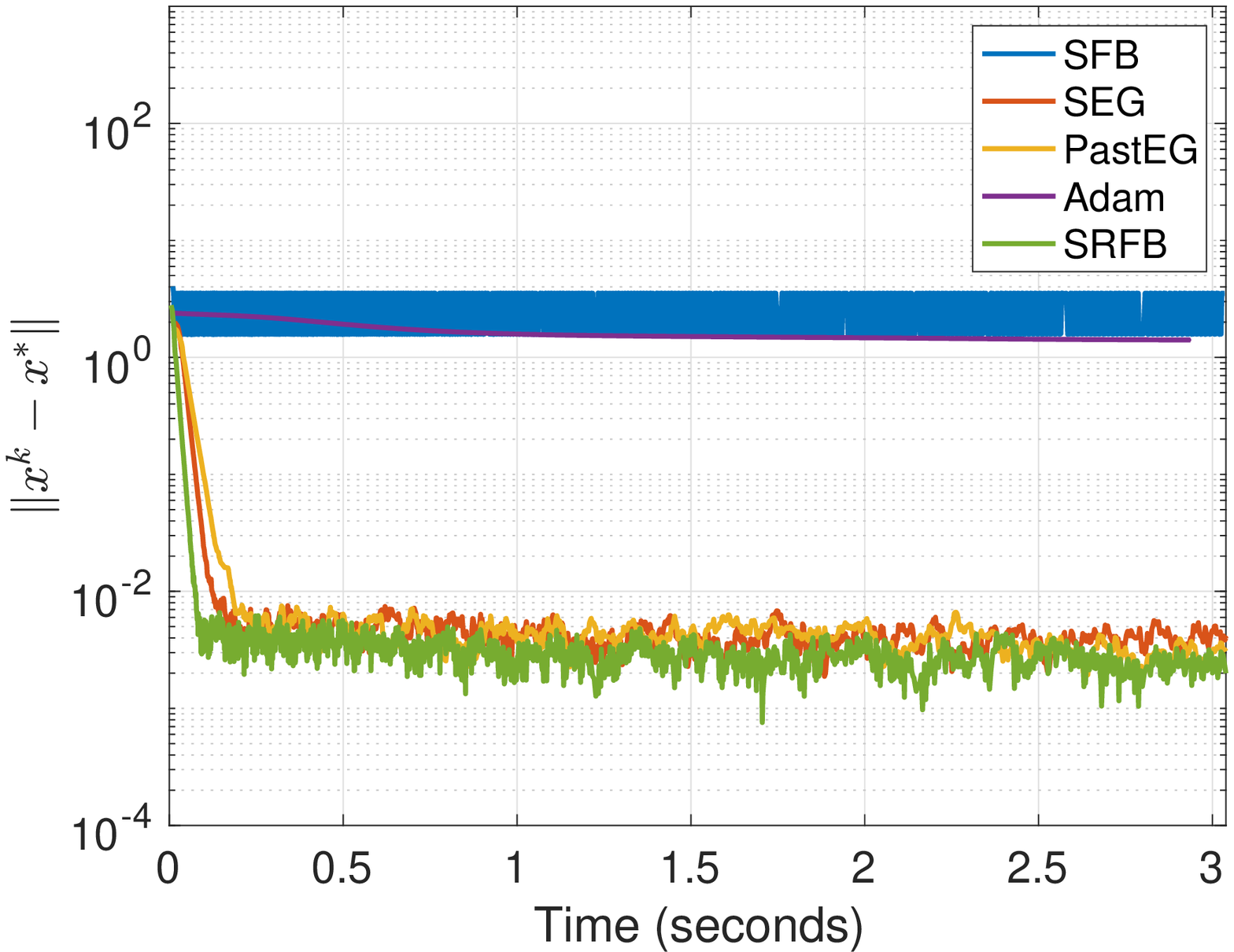}
\caption{Computational complexity.}\label{plot_sec}
\end{subfigure}
\caption{Example \ref{ex_academic}.}
\end{figure}

\subsection{Classic GAN zero-sum game}\label{ex_game}
A classic cost function for the zero-sum game \cite{goodfellow2016} is
$$
\min _{x_\text{g}} \max _{x_\text{d}}-\log (1+e^{-x_\text{d}^{\top} \omega})-\log (1+e^{x_\text{d}^{\top} x_\text{g}}).
$$
This cost function is hard to optimize because it is concave-concave \cite{gidel2019}.
Here we take $\omega=-2,$ thus the equilibrium is $(x_\text{g}, x_\text{d})=(-2,0) .$ In Figure \ref{plot_toy_sol}, \ref{plot_toy_ave} and \ref{plot_toy_sec}, we show the distance from the solution, the distance of the average from the solution, and the computational cost respectively. 
Interestingly, all the considered algorithms converge even if there are no theoretical guarantees.

\begin{figure}
\begin{subfigure}{\columnwidth}
\centering
\includegraphics[width=.9\textwidth]{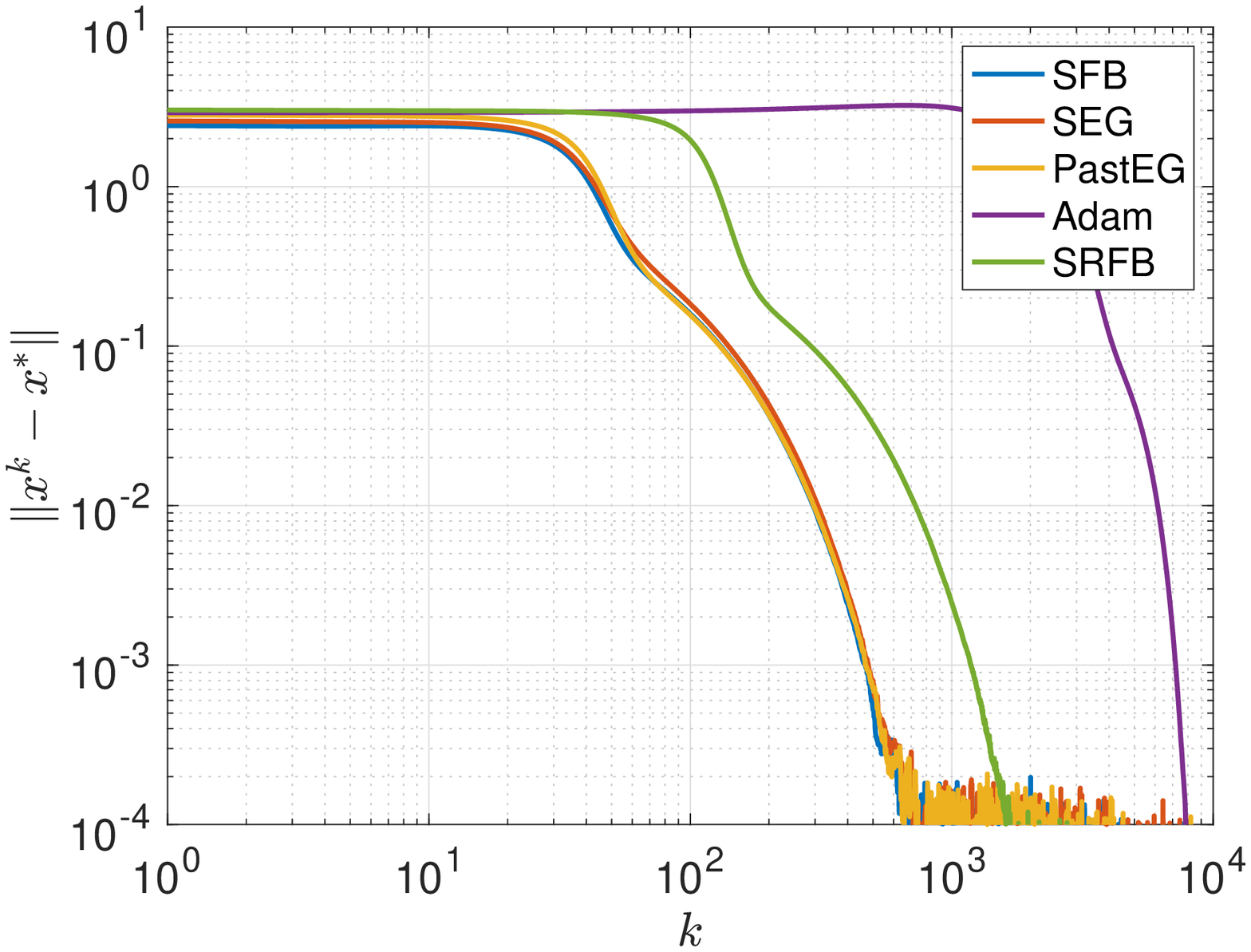}
\caption{Relative distance from the solution.}\label{plot_toy_sol}
\end{subfigure}
\begin{subfigure}{\columnwidth}
\centering
\includegraphics[width=.9\textwidth]{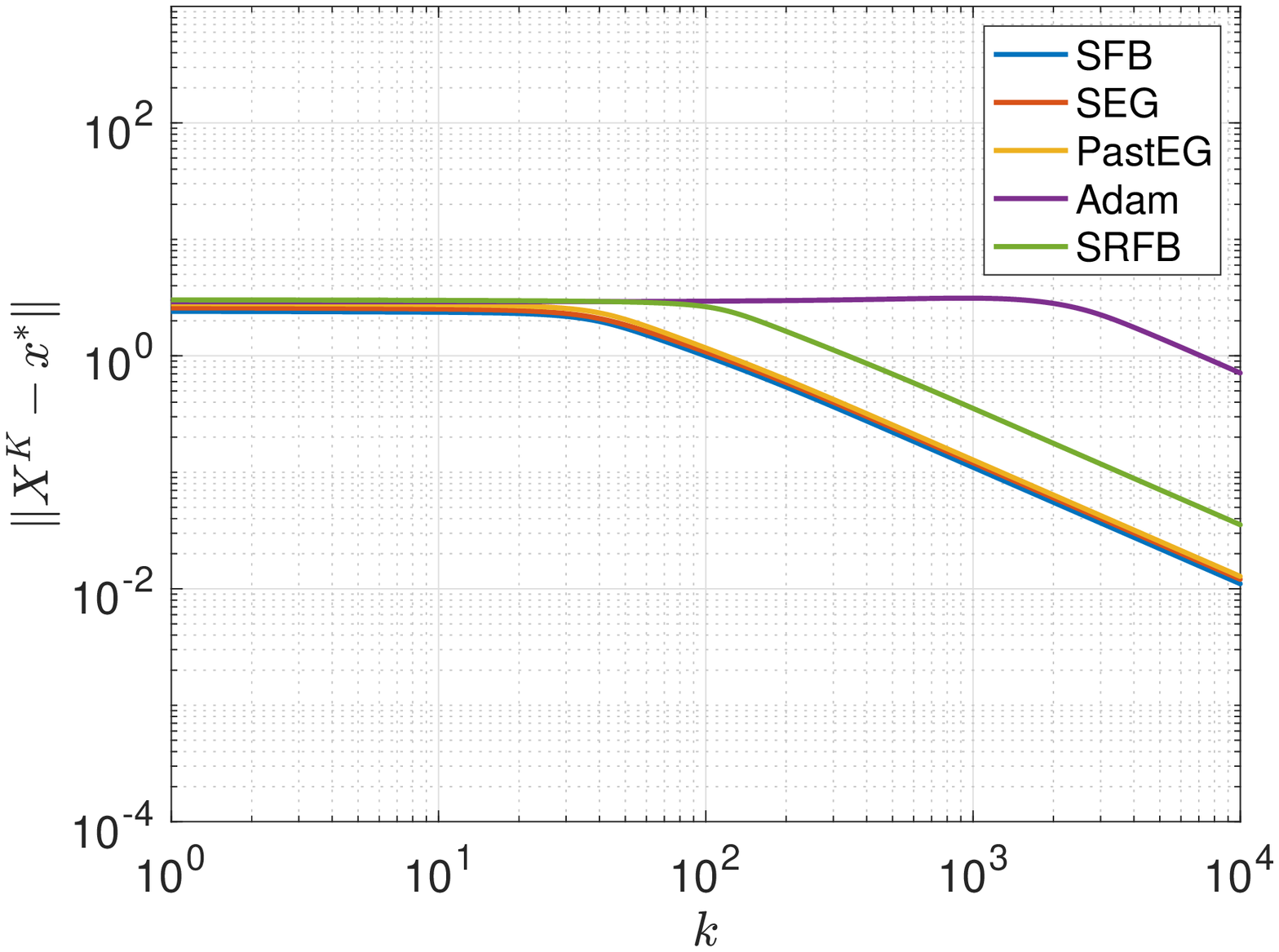}
\caption{Relative distance of the average from the solution.}\label{plot_toy_ave}
\end{subfigure}
\begin{subfigure}{\columnwidth}
\centering
\includegraphics[width=.9\textwidth]{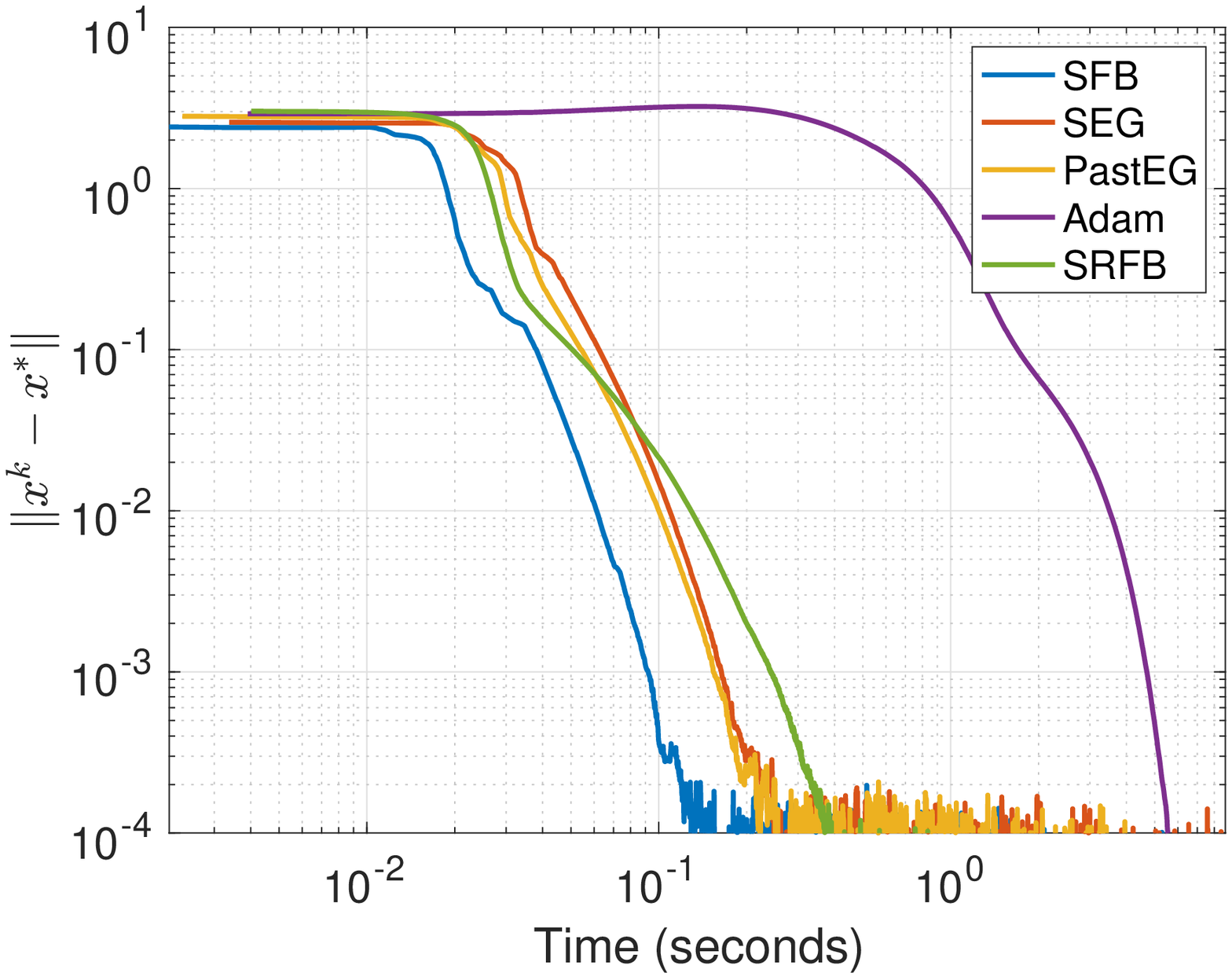}
\caption{Computational complexity.}\label{plot_toy_sec}
\end{subfigure}
\caption{Example \ref{ex_game}.}
\end{figure}

\section{Conclusion}
The stochastic relaxed forward--backward algorithm can be applied to Generative Adversarial Networks. Given a fixed mini-batch, under monotonicity of the pseudogradient, averaging can be considered to reach a neighbourhood of the solution. On the other hand, if a huge number of samples is available, under the same assumptions, convergence to the exact solution holds.

\appendices
\section{Preliminary results}

We here recall some preliminary results.



The Robbins-Siegmund Lemma is used in literature to prove a.s. convergence of sequences of random variables.
\begin{lemma}[Robbins-Siegmund Lemma, \cite{RS1971}]\label{lemma_RS}
Let $\mc F=(\mc F_k)_{k\in\NN}$ be a filtration. Let $\{\alpha_k\}_{k\in\NN}$, $\{\theta_k\}_{k\in\NN}$, $\{\eta_k\}_{k\in\NN}$ and $\{\chi_k\}_{k\in\NN}$ be non negative sequences such that $\sum_k\eta_k<\infty$, $\sum_k\chi_k<\infty$ and let
$$\forall k\in\NN, \quad \EE[\alpha_{k+1}|\mc F_k]+\theta_k\leq (1+\chi_k)\alpha_k+\eta_k \quad a.s.$$
Then $\sum_k \theta_k<\infty$ and $\{\alpha_k\}_{k\in\NN}$ converges a.s. to a non negative random variable.\fineass
\end{lemma}

The next lemma collects some properties that follow from the definition of the SRFB algorithm.
\begin{lemma}\label{lemma_algo}
Given Algorithm \ref{algo_i}, the following hold.
\begin{enumerate}
\item $\bs x^k-\bar {\bs x}^{k-1}=\frac{1}{\delta}(\bs x^k-\bar {\bs x}^k)$
\item $\bs x^{k+1}-\bs x^*=\frac{1}{1-\delta}(\bar{\bs x}^{k+1}-\bs x^*)-\frac{\delta}{1-\delta}(\bar{\bs x}^k-\bs x^*)$
\item $\frac{\delta}{(1-\delta)^2}\norm{\bar{\bs x}^{k+1}-{\bs x}^k}^2=\delta\norm{\bs x^{k+1}-{\bs x}^k}^2$
\end{enumerate}
\end{lemma}
\begin{proof}
Straightforward from Algorithm \ref{algo_i}.
\end{proof}

\section{Proof of Theorem \ref{theo_ave}}\label{app_theo_ave}

\begin{proof}[Sketch]
Using the fact that the projection operator is quasi firmly non expansive \cite[Def. 4.1, Prop. 4.16]{bau2011}, applying Lemma 2, the assumptions involved and some norm properties as the Young's inequality, it is possible to prove that, summing over all the iterations and following \cite[Th. 2]{gidel2019}, it holds that
\begin{equation}
\begin{aligned}
2S\langle\FF(\bs x^*)&,\bs X^K-\bs x^*\rangle\leq\textstyle\sum\nolimits_{k=1}^K\lambda^2\EEk{\normsq{\epsilon^k}}+\\
&+\textstyle\frac{2-\delta}{1-\delta}\normsq{\bar{\bs x}^0-\bs x^*}+\delta\normsq{\bs x^{0}-\bar{\bs x}^{-1}}+\\
&\textstyle+2\sum\nolimits_{k=1}^K\lambda^2\EEk{\normsq{F^{\textup{SA}}(\bs x^k,\xi^k)}}+\\
&\leq\textstyle\frac{2-\delta^2}{1-\delta}R+(2B^2+\sigma^2)\sum\nolimits_{k=1}^K\lambda^2,
\end{aligned}
\end{equation}
where $S=K\lambda$ and $\sum\nolimits_{k=1}^K\lambda^2=K\lambda^2$. Then, it follows that 
\begin{equation*}
\langle\FF(\bs x^*),\bs X^K-\bs x^*\rangle\leq \frac{cR}{K\lambda}+(2B^2+\sigma^2)\lambda.\vspace{-.6cm}
\end{equation*}
\end{proof}

\section{Proof of Theorem \ref{theo_SAA}}\label{app_theo_SAA}

\begin{proof}[Sketch]
Since the projection operator, by \cite[Prop. 12.26]{bau2011}, satisfies, for all $x,y\in C$, $\bar x=\op{proj}_C(x)\Leftrightarrow\langle \bar x-x, y-\bar x\rangle \geq 0$ where $C$ is a nonempty closed convex set, the proof follows similarly to \cite[Th. 1]{malitsky2019}. Specifically, using the monotonicity assumption and Lemma \ref{lemma_algo}, in combination with Cauchy-Schwartz and Young's inequalities and the cosine rule, we exploit the definition of the stochastic error and the fact that (using the definition of residual and firmly non expansiveness of the projection)
\begin{equation*}
\op{res}(\bs x^k)^2\leq2\normsq{\bs x^k-\bs x^{k+1}}+4\normsq{\bar{\bs x}_k-\bs x^k}+\lambda^2\normsq{\epsilon_k}.
\end{equation*}
Thus, it holds that, taking the expected value and using Remark \ref{remark_error}
\begin{equation*}\label{stepNk10}
\begin{aligned}
&\EEk{\textstyle\frac{1}{1-\delta}\norm{\bar{\bs x}^{k+1}-\bs x^*}^2}+\\
&+\EEk{\textstyle\left(\frac{1}{2\delta}-\ell\lambda-\lambda\right)\norm{\bs x^k-\bs x^{k+1}}^2}\leq\\
&\leq(\textstyle\frac{1}{1-\delta}\norm{\bs x^*-\bar {\bs x}^k}^2+\ell\lambda\norm{\bs x^k-\bs x^{k-1}}^2-\frac{1}{4\delta}\op{res}(\bs x^k)^2+\\
&+\textstyle\frac{2\lambda C\sigma}{N_k}+\frac{2\lambda C\sigma}{N_{k-1}}+\frac{\lambda}{\delta}\frac{C\sigma}{N_k}-\textstyle\frac{1}{\delta}\norm{\bs x^k-\bar {\bs x}^k}^2\\
\end{aligned}
\end{equation*}
To use Lemma \ref{lemma_RS}, let
$\alpha_k= \frac{1}{1-\delta}\norm{\bs x^*-\bar {\bs x}^k}^2+\ell\lambda\norm{\bs x^k-\bs x^{k-1}}^2,$
$\theta_k=\frac{1}{\delta}\norm{\bs x^k-\bar {\bs x}^k}^2+\frac{1}{4\delta}\op{res}(\bs x^k)^2,$
$\eta_k=\frac{2\lambda C\sigma}{N_k}+\frac{2\lambda C\sigma}{N_{k-1}}+\frac{\lambda}{\delta}\frac{C\sigma}{N_k}.$
Applying the Robbins Siegmund Lemma we conclude that $\alpha_k$ converges and that $\sum\theta_k$ is summable. This implies that the sequence $\{\bar {\bs x}^k\}$ is bounded and that $\|\bs x^k-\bar {\bs x}^k\|\to 0$. Therefore $\{\bs x^k\}$ has at least one cluster point $\tilde{\bs x}$. Moreover, since $\sum\theta_k<\infty$, $\op{res}(\bs x^k)^2\to0$ and $\op{res}(\tilde{\bs x})^2=0$.
\end{proof}

\bibliographystyle{IEEEtran}
\bibliography{Biblio}

\end{document}